\documentclass[submission,copyright,creativecommons]{eptcs}
\providecommand{\event}{CAPNS2018} 

\usepackage{color}
\usepackage{amsmath,amsthm}
\usepackage{amsfonts,amssymb}
\usepackage{mathrsfs}
\usepackage{hyperref}

\usepackage{tikz}
\usepackage{tikz-cd}
\usepackage{tikz}
\usepackage{pgfplots}
\usetikzlibrary{trees}
\usetikzlibrary{topaths}
\usetikzlibrary{decorations.pathmorphing}
\usetikzlibrary{decorations.markings}
\usetikzlibrary{matrix,backgrounds,folding}
\usetikzlibrary{chains,scopes,positioning,fit}
\usetikzlibrary{arrows,shadows}
\usetikzlibrary{calc} 
\usetikzlibrary{chains}
\usetikzlibrary{shapes,shapes.geometric,shapes.misc}

\pgfdeclarelayer{edgelayer}
\pgfdeclarelayer{nodelayer}
\pgfsetlayers{background,edgelayer,nodelayer,main}

\tikzstyle{dot}=[circle,fill=black,draw=black]
\tikzstyle{none}=[inner sep=0pt]
\tikzstyle{every loop}=[]

\tikzstyle{(null)}=[]
\tikzstyle{plain}=[]

\tikzstyle{blank}=[inner sep=0pt]
\tikzstyle{box}=[rectangle, minimum size = 0.5cm,fill=white,draw=black]
\tikzstyle{small_node}=[inner sep=0pt, minimum size=0.2cm,circle,fill=white,draw=black]
\tikzstyle{square}=[thick, minimum size=0.3cm,rectangle,fill=white,draw=black]

\tikzstyle{downtri}=[regular polygon,regular polygon sides=3,shape border rotate=180,fill=white,draw=black]
\tikzstyle{uptri}=[regular polygon,regular polygon sides=3,shape border rotate=0,fill=white,draw=black]

\tikzstyle{morph}=[->,draw=black,line width=0.600]
\tikzstyle{arrow}=[-,draw=black,postaction={decorate},decoration={markings,mark=at position .5 with {\arrow{>}}},line width=2.000]
\tikzstyle{tick}=[-,draw=black,postaction={decorate},decoration={markings,mark=at position .5 with {\draw (0,-0.1) -- (0,0.1);}},line width=2.000]

\usepackage[all,2cell,cmtip]{xy}
\UseTwocells

\newcommand{\define}[1]{{\bf \boldmath{#1}}}

\newcommand{\cat}[1]{\textup{\textsf{#1}}}
\newcommand{\FVect}{\cat{FVect}}

\newcommand{\DisCoCat}{\cat{DisCoCat}}

\newcommand{\Dict}{\cat{Dict}}

\newcommand{\maps}{\colon}
\newcommand{\C}{\mathscr{C}}
\newcommand{\Cat}{\cat{Cat}}
\newcommand{\lang}[1]{\ensuremath{\textit{#1}}}
\newcommand{\im}{\mathrm{Im} \;}
\newcommand{\PS}{\mathtt{PS}}
\newcommand{\vect}[1]{\mathbf{#1}}

\theoremstyle{definition}
\newtheorem{thm}{Theorem}[section]
\newtheorem{defn}[thm]{Definition}
\newtheorem{prop}[thm]{Proposition}
\theoremstyle{remark}
\newtheorem{ex}[thm]{Example}
\newtheorem{rmk}[thm]{Remark}

\title{Translating and Evolving: Towards a Model of Language Change in DisCoCat}

\author{Tai-Danae Bradley \institute{Graduate Center, CUNY} \email{tbradley@gradcenter.cuny.edu} \and Martha Lewis \institute{ILLC, University of Amsterdam} \email{m.a.f.lewis@uva.nl} \and Jade Master \institute{Dept. Mathematics, UC Riverside} \email{jmast003@ucr.edu} \and Brad Theilman \institute{Gentner Lab, UC San Diego} \email{btheilma@ad.ucsd.edu}}

\def\titlerunning{Translating and Evolving}
\def\authorrunning{T. Bradley, M. Lewis, J. Master, B. Theilman}

\begin{document}
\maketitle
\begin{abstract}
The categorical compositional distributional (DisCoCat) model of meaning developed by \cite{coecke2010} has been successful in modeling various aspects of meaning. However, it fails to model the fact that language can change. We give an approach to DisCoCat that allows us to represent language models and translations between them, enabling us to describe translations from one language to another, or changes within the same language. We unify the product space representation given in \cite{coecke2010} and the functorial description in \cite{kartsaklis2013reasoning}, in a way that allows us to view a language as a catalogue of meanings. We formalize the notion of a lexicon in DisCoCat, and define a dictionary of meanings between two lexicons. All this is done within the framework of monoidal categories. We give examples of how to apply our methods, and give a concrete suggestion for compositional translation in corpora.
\end{abstract}

\section{Introduction}
Language allows us to communicate, and to compose words in a huge variety of ways to obtain different meanings. It is also constantly changing. The compositional distributional model of \cite{coecke2010} describes how to use compositional methods within a vector space model of meaning. However, this model, and others that are similar \cite{baroni2010, maillard2014}, do not have a built in notion of language change, or of translation between languages.

In contrast, many statistical machine translation systems currently use neural models, where a large network is trained to be able to translate words and phrases \cite{mikolovtranslation, gao2014}. This approach does not make use of the grammatical structure which allows you to build translations of phrases from the translations of individual words. In this paper we define a notion of translation between two compositional distributional models of meaning which constitutes a first step towards unifying these two approaches.

Modeling translation between two languages also has intrinsic value, and doing so within the DisCoCat framework means that we can use its compositional power. In section \ref{catlangmodel}, we provide a categorical description of translation between two languages that encompasses both updating or amending a language model and translating between two distinct natural languages.

 In order to provide this categorical description, we must first introduce some preliminary concepts. In section \ref{productspace} we propose a unification of the product space representation of a language model of \cite{coecke2010} and the functorial representation of \cite{kartsaklis2013reasoning}. This allows us to formalize the notion of lexicon in section \ref{lexicons} which had previously been only loosely defined in the DisCoCat framework. We then show how to build a dictionary between two lexicons and give an example showing how translations can be used to model an update or evolution of a compositional distributional model of meaning. In section \ref{EngSpan} we give a concrete suggestion for automated translation between corpora in English to corpora in Spanish.
\section{Background}
\subsection{Categorical Compositional Distributional Semantics}
\label{sec:DisCo}
\em Categorical compositional distributional models \em \cite{coecke2010} successfully exploit the compositional structure of natural language in a principled manner, and have outperformed other approaches in Natural Language Processing (NLP) \cite{grefenstette2011, kartsaklis2013}.  The approach works as follows.   A mathematical formalization of grammar is chosen, for example \em Lambek's pregroup grammars \em \cite{lambek2001},  although the approach is equally effective with other categorial grammars~\cite{Coecke2013}.  Such a categorial grammar allows one to verify whether a phrase or a sentence is grammatically well-formed by means of a computation that establishes the overall grammatical type,  referred to as \em a type reduction\em.  The meanings of \em individual \em words are established using a distributional model of language, where they are described as vectors of co-occurrence statistics derived automatically from corpus data~\cite{lund1996}.  The categorical compositional distributional programme unifies these two aspects of language in a compositional model where grammar mediates composition of meanings. This  allows us to derive the meaning of sentences from their grammatical structure, and the meanings of their constituent words. The key insight that allows this approach to succeed is that both pregroup grammars and the category of vector spaces carry the same abstract structure~\cite{coecke2010},  and the same holds for other categorial grammars since they typically have a weaker categorical structure.

The categorical compositional approach to meaning uses the notion of a \emph{monoidal category}, and more specifically a \emph{compact closed} category to understand the structure of grammar and of vector spaces. For reasons of space, we do not describe the details of the compositional distributional approach to meaning. Details can be found in \cite{coecke2010, kartsaklis2013reasoning}, amongst others. We note only that instead of using a pregroup as our grammar category, we use the free compact closed category $J = \C(\mathscr{B})$ generated over a set of types $\mathscr{B}$, as described in \cite{prellerlambek2007, preller2013}.

\section{Translating and Evolving}

The categorical model has proved successful in a number of natural language processing tasks \cite{grefenstette2011, kartsaklis2013}, and is flexible enough that it can be extended to include ambiguity \cite{piedeleu2015} and changes of the semantic category \cite{bolt2017, almehairi2016}. These formalisms have allowed for connections between semantic meanings. By representing words as density matrices, a variant of L\"{o}wner ordering has been used to measure the degree of entailment between two words \cite{balkir2015, bankova2016}. A more simple notion of similarity has been implemented in the distributional model by using dot product \cite{coecke2010}. However, these notions of similarity are not built into the formalism of the model. This section defines the notion of a categorical language model which keeps track of internal relationships between semantic meanings. 

So far the implementation of these models has been static. In this section, we define a notion of translation which comprises a first step into bringing dynamics into these models of meaning. We show how a language model can be \emph{lexicalized}, i.e. how vocabulary can be attached to types and vectors and introduce a category of lexicons and translations between them. This allows dictionary between phrases in one language model and the phrases in another.

\subsection{Categorical Language Models and Translations}\label{catlangmodel}
\begin{defn}\label{languagemodel}Let $J$ be a category which is freely monoidal on some set of grammatical types. A \define{distributional categorical language model} or \define{language model} for short is a strong monoidal functor 
	\[ F \maps (J, \cdot) \to (\FVect, \otimes) \]
\end{defn}
\noindent If $J$ is compact closed then the essential image of $F$ inherits a compact closed structure. All of the examples we consider will use the canonical compact closed structure in $\FVect$. However, this is not a requirement of the general approach, and other grammars that are not compact closed my be used, such as Chomsky grammars \cite{hedges2016} or Lambek monoids \cite{Coecke2013}.

Distributional categorical language models do not encapsulate everything about a particular language. In fact, there are many possible categorical language models for the same language and there is a notion of translation between them.

	\begin{defn}\label{translation}
		A \define{translation} $T=(j, \alpha)$ from a language model $F \maps J \to \FVect$ to a language model $F' \maps J' \to \FVect$ is a monoidal functor $j \maps J \to J'$ and a monoidal natural transformation $\alpha \maps F \Rightarrow F' \circ j$. Pictorially, $(j,\alpha)$ is the following 2-cell 
		\[ 
		\xymatrix@=0.5em{
			J \ar[rrr]^{F}  \ar[ddd]_{j} &  & & \FVect \\
			&  \mathbin{\rotatebox[origin=c]{-90}{$\Rightarrow$}} 
		\;\alpha & \\
			\\
			J' \ar[uuurrr]_{F'} & &}
		\]	 
		Given another language model $F'' \maps J'' \to \FVect$ and a translation $T'= (j', \alpha')$ the composite translation is computed pointwise. That is, $T'\circ T$ is the translation $(j'\circ j, \alpha' \circ \alpha)$ where $\alpha'\circ \alpha$ is the vertical composite of the natural transformations $\alpha$ and $\alpha'$.	
	\end{defn}
\begin{defn}
		 Let $\define{\DisCoCat}$ be the category with distributional categorical language models as objects, translations as morphisms, and the composition rule described above.
\end{defn}
This category allows us to define ways of moving between languages. The most obvious application of this is that of translation between two languages such as English and Spanish. However, the translation could also be from a simpler language to a more complex language, which we think of as learning, and it could be within a shared language, where we see the language evolving.

\subsection{The Product Space Representation}\label{productspace}
In \cite{coecke2010} the \emph{product space representation} of language models was introduced as a way of linking grammatical types with their instantiation in $\FVect$. The idea is that the meaning computations take place in the category $J \times \FVect$ where $J$ is a pregroup or free compact closed category. Let $p$ be a sentence, that is a sequence of words $\{w_i\}$ whose grammatical types reduce to the sentence type $s$. To compute the meaning of $p$ you:
\begin{itemize}
	\item Determine both the grammatical type $g_i$ and distributional meaning $\vect{v}_i$ in $V_{g_i}$ where $V_{g_i}$ is a meaning space for the grammatical type $g_i$.
	\item Using the monoidal product and tensor product, obtain the element $g_1 \cdot \ldots \cdot g_n \in J$ and $\vect{v}_1 \otimes \cdots \otimes \vect{v_n} \in V_{g_1} \otimes \ldots \otimes V_{g_n}$.
	\item Let $r \maps g_1 \cdot \cdots \cdot g_n \to s$ be a type reduction in $J$. There is a linear transformation $M \maps V_{g_1} \otimes \cdots \otimes V_{g_n} \to V_{s}$ given by matching up the compact closed structure in $J$ with the canonical compact closed structure in $\FVect$. Apply $M$ to the vector $\vect{v}_1 \otimes \ldots \otimes \vect{v}_n$ to get the distributional meaning of your sentence.
\end{itemize}

The product space $J \times \FVect$ provides a setting in which the meaning computations take place but it does not contain all of the information required to compute compositional meanings of sentences. To do this requires an assignment of every grammatical type to a vector space and every type reduction to a linear transformation in a way which preserves the compact closed structure of both categories. This suggests that there is a compact closed functor $F \maps J \to \FVect$ lurking beneath this approach. With this in mind we introduce a new notion of the product space representation using the Grothendieck construction \cite{elephant}. In order to use the Grothendieck construction we first need to interpret vector spaces as categories. 

In this paper, we will do this in two mostly trivial ways which do not take advantage of the vector space structure in $\FVect$. The first way we will turn vector spaces into categories is via the discrete functor 
\[D \maps \FVect \to \Cat\]

 $D$ assigns each vector space $V$ to the discrete category of its underlying set. For a linear transformation $M\maps V \to W$, $D(M)$ is the unique functor from $D(V)$ to $D(W)$ which agrees with $M$ on the elements of $V$. 

There is another way to generate free categories from sets.
\begin{defn}\label{C}
	Let $V$ be a finite dimensional real vector space. Then, the \define{free chaotic category} on $V$, denoted $C(V)$, is a category where
	\begin{itemize}
		\item objects are elements of $V$ and,
		\item for all $\vect{u}$,$\vect{v}$ in $V$ we include a unique arrow $d(\vect{u},\vect{v}) \maps \vect{u} \to \vect{v}$ labeled by the Euclidean distance $d(\vect{u},\vect{v})$ between $\vect{u}$ and $\vect{v}$.
	\end{itemize}
	This construction extends to a functor $C \maps \FVect \to \Cat$. For $M \maps V \to W$, define $C(M) \maps C(V) \to C(W)$ to be the functor which agrees with $M$ on objects and sends arrows $d(\vect{u},\vect{v})$ to $d(M\vect{u},M\vect{v})$.
\end{defn}
The morphisms in $C(V)$ for a vector space $V$ allow us to keep track of the relationships between different words in $V$.

We now give a definitions of the product space representation in terms of the Grothendieck construction which depends on a choice of functor $K \maps \FVect \to \Cat$.
\begin{defn}\label{def:PS}
	Let $F \maps J \to \FVect$ be a language model and let $K \maps \FVect \to \Cat$ be a faithful functor. The \define{product space representation} of $F$ with respect to $K$, denoted $\PS_K(F)$, is the Grothendieck construction of $K \circ F$. Explicitly, $\PS_K(F)$ is a category where 
	\begin{itemize}
		\item an object is a pair $(g, \vect{u})$ where $g$ is an object of $J$ and $\vect{u}$ is an object of $K \circ F(g)$
		\item a morphism from $(g, \vect{u})$ to $(h,\vect{v})$ is a tuple $(r,f)$ where $r \maps g \to h$ is a morphism in $J$ and $f \maps K\circ F(r)(\vect{u}) \to \vect{v}$ is a morphism in $K \circ F(h)$
		\item the composite of $(r',f') \maps (g,\vect{u}) \to (h,\vect{v})$ and $(r,f) \maps (h,\vect{v}) \to (i,\vect{x})$ is defined by 
		\[ (r,f) \circ (r' ,f') = ( r \circ r', f \circ (K\circ F)(r)(f') ) \]
	\end{itemize}
\end{defn}

\begin{rmk}
	Because $K$ is faithful, it is an equivalence of categories onto its essential image in $\Cat$. Because monoidal structures pass through equivalences $K \circ F \maps J \to \im  K \circ F$ is a monoidal functor where $\im K \circ F$ denotes the essential image of $K \circ F$.
\end{rmk}
When $K$ is equal to the discrete category functor $D$, then the product space representation is the \define{category of elements of $F$}. This is a category where
\begin{itemize}
	\item objects are pairs $(g,\vect{u})$ where $g$ is a grammatical type and $\vect{u}$ is a vector in $F(g)$.
	\item a morphism $r \maps (g,\vect{u}) \to (h,\vect{v})$ is a type reduction $r \maps g \to h$ such that $F(r)(\vect{u}) = \vect{v}$
\end{itemize}

In this context we can compare the product space representation in Definition \ref{def:PS} with the representation introduced in \cite{coecke2010} to see that they are not the same. One difference is that $\PS_D (F)$ only includes the linear transformations that correspond to type reductions and not arbitrary linear transformations. This narrows down the product space representation to a category that characterizes the meaning computations which can occur. Also, the meaning reductions in $\PS_D (F)$ correspond to morphisms in the product space representation whereas before they occurred within specific objects of the product space. Using this definition of the product space representation, we are able to formally introduce a lexicon into the model and understand how these lexicons are affected by translations.

When $K = C$ as in Definition \ref{C} the product space representation is as follows:
\begin{itemize}
\item objects are pairs $(g,\vect{u})$ where $g$ is a grammatical type and $\vect{u}$ is a vector in $F(g)$.
\item a morphism $(r, d) \maps (g,\vect{u}) \to (h,\vect{v})$ is:
\begin{itemize}
\item a type reduction $r \maps g \to h$
\item a positive real number ${d \maps C\circ F (r) (\vect{u}) \to \vect{v}}$
\end{itemize}
\end{itemize}
Now, objects in $\PS_C(F)$ are pairs of grammatical types and \emph{vectors}, rather than vector spaces. We can therefore see $\PS_C(F)$ as a catalogue of all possible meanings associated with grammatical types. The linear transformations available in $\PS_C(F)$ are only those that are derived from the grammar category.

\begin{prop}[$\PS_K(F)$ is monoidal]
	For $K=C$ and $K=D$,
	$\PS_K(F)$ is a monoidal category with monoidal product given on objects by \[(g,\vect{u}) \otimes (h,\vect{v}) = (g\cdot h,\Phi_{g,h} (\vect{u} \otimes \vect{v})  )\]
	and on morphisms by 
	\[ (r,f) \otimes (r',f') = (r \cdot r', \Phi_{g,h}(f \otimes  f'))\]
	where $\Phi_{g,h} \maps K \circ F (g) \vect{\otimes} K \circ F(h) \to K \circ F(g \cdot h)$ is the natural isomorphism included in the data of the monoidal functor $K \circ F$.
\end{prop}

\begin{proof}
	Adapted from Theorem 38 of \cite{BaezMoeller}.
\end{proof}
The fact that $\PS_K(F)$ is monoidal enables us to use the powerful graphical calculus available for monoidal categories. Previously the monoidal graphical calculus has only been used to pictorially reason about grammatical meanings. Because the elements of the product space representation represent both the syntactic and semantic meaning, this proposition tells us that we can reason graphically about the entire meaning of our phrase.

The product space construction also applies to translations:

\begin{prop}[Translations are monoidal]\label{int} 
	Let $K \maps \FVect \to \Cat$ be a fully faithful functor.
	Then there is a functor $\PS_K \maps \DisCoCat \to \cat{MonCat}$, where \cat{MonCat} is the category where objects are monoidal categories and morphisms are strong monoidal functors, 
	 that sends
	\begin{itemize}
		\item language models $F \maps J\to \Cat$ to the monoidal category $\PS_K (F)$
		\item translations $T=(j, \alpha)$ to the strong monoidal functor $\PS_K(T) \maps \PS_K (F) \to \PS_K (F')$  where the functor $\PS_K (T)$ acts as follows:
		\begin{itemize}
			\item On objects, $\PS_K (T)$ sends $(g,\vect{u})$ to $(j(g),\alpha_g \vect{u})$.
			\item Suppose $(r,f) \maps (g,\vect{u}) \to (h,\vect{v})$ is a morphism in $\PS_K (F)$ so that $r \maps g \to h$ is a reduction in $J$ and $f \maps F(r)(\vect{u}) \to \vect{v}$  is a morphism in $F(h)$. Then $\PS_K (T)$ sends $(r,f)$ to the pair $(j(r), \alpha_h \circ f)$.
		\end{itemize} 	
	\end{itemize}
\end{prop}

\begin{proof}
Adapted from Theorem 39 in \cite{BaezMoeller}.
\end{proof}
\subsection{Lexicons}\label{lexicons}
 Using our definition of the product space representation, we are able to formally introduce a lexicon into the model and describe how these lexicons are affected by translations. In what follows we fix $K = C$ in all product space constructions and denote $\PS_C (F)$ as $\PS(F)$. We also use the notation $F_C$ for $C \circ F$.

	\begin{defn}
\label{def:lexicon}
	Let $F$ be a categorical language model and let $W$ be a finite set of words, viewed as discrete category. Then a \define{lexicon for $F$} is a functor $ \ell \maps W\to \PS (F)$. 
This corresponds to a function from $W$ into the objects of $\PS(F)$. 
	\end{defn}
Lexicons can be extended to arbitrary phrases in the set of words $W$. Phrases are finite sequences of words $v_1 \ldots v_n \in W^*$ where $W^*$ is the free monoid on $W$. The function $\ell$ assigns to each $v_i \in W$ the pair $(g_i,\vect{v_i})$ corresponding to its grammatical type $g$ and its semantic meaning $\vect{v_i} \in F(g_i)$. Because $W^*$ is free, this defines a unique object in $\PS(F)$:
\[
(g, \vect{v}) := \otimes_{i=1}^{n} \ell (v_i)= (g_1, \vect{v}_1) \otimes  \ldots \otimes (g_n , \vect{v}_n) =  (g_1 \cdots g_n,\vect{v}_1 \otimes \ldots \otimes \vect{v}_n) 
\]
where $g_i$ is the grammatical type of $v_i$ and $\vect{v}_i$ is the semantic meaning of $v_i$ for $ i \in \{1, ..., n\}$. The extension of $\ell$ to $W^*$ will be denoted by 
\[l^* \maps W^* \to \PS(F).\]



\begin{ex}
Let $J = \C(\{n,s\})$ be the free compact closed category on the grammatical types of nouns and sentences. Then, for the phrase $\lang{u} = \lang{Rose likes Rosie}$, a lexicon for $F$ gives the unique element
\[
\ell(\lang{u}) = (g, \vect{u}) = (n,\vect{Rose})\otimes (n^rsn^l,\vect{likes})\otimes (n,\vect{Rosie})\]
In $\PS(F)$, the grammar type $n^r s n^l$ reduces 
to $s$ via the morphism $r=\epsilon_n\; 1_s \;\epsilon_n$ and so we get a reduction  \[(n,\vect{Rose})\otimes (n^rsn^l,\vect{likes})\otimes (n,\vect{Rosie})\overset{(r,F(r))}{\longrightarrow}(s,\vect{Rose\;likes \;Rosie})\]
\end{ex}
%
%
%
To fully specify a translation between two lexicons it is not necessary to manually match the words in each corpora. This is because a relation between the phrases in the corpora can be derived from a translation between the language models.
\begin{defn}\label{def:dictionary}
	
	Let $\ell \maps W \to \PS (F)$ and $m \maps V \to \PS (G)$ be lexicons and let $T$ be a translation from $F$ to $G$. Then, the \define{$F$-$G$ dictionary} with respect to $T$ is the comma category 
	\[(\PS (T) \circ \ell^* ) \downarrow m^*\]
	denoted by $\Dict_T$. Since $W$ and $V$ are discrete categories, $(\PS (T) \circ \ell^*) \downarrow m^*$ is a set of triples $(p,(r,d),q)$ 
where $p \in W^*$, $q \in V^*$ and $(r,d)\maps (\PS(T) \circ \ell) (p) \to m(q)$  
is a morphism in $\PS(G)$. Explicitly, let 
\[
 \ell(p)= (g,\vect{p})\quad\text{ and }\quad m(q) = (g',\vect{q})\] 
then $(r,d)$ is 
\begin{itemize}
\item a type reduction $r \maps j(g) \to g'$ in the grammar category $J$
\item a morphism $d$ in $C \circ G (g')$. Recall from Definition \ref{C} that this corresponds to a real number 
$d(\vect{p}', \vect{q})$ denoting the distance between $\vect{p}'$ and $\vect{q}$ in  $G(g')$. Here, $\vect{p}'$  is the vector that results from applying the translation and any grammatical reductions. Namely, 
$\vect{p}' = (C \circ G (r) \circ \alpha_g)\vect{p}$ 
i.e., we firstly translate the vector $\vect{p}$ into $\PS(F')$, then apply the linear map corresponding to the reduction $r$, and finally send the resulting vector to its corresponding object in the chaotic category.
\end{itemize}
\end{defn}
$\Dict_{T}$ and  allows us to keep track of the distances between phrases in $W$ to phrases in $W'$ in a compositional way; similarities between phrases are derived from similarities between the constituent words. 

Let $k$ be a positive real number. Then define $\Dict_{T, k}$ to be the relation which pairs two words in $\Dict_{T}$ if the distance $d$ between their semantic meanings is less than or equal to $k$. The purpose of this is to say that we are interested in pairs of words and phrases which do not have to be identical, but whose meaning is sufficiently close.


\begin{ex}[Syntactic simplification]\label{ex:pluralsing} We give an example of a translation from a language with several noun types, accounting for singular $n_s$ and plural $n_p$ nouns to a language with one noun type. To start, suppose $W^*$ is the free monoid on the set $\{\lang{Rosie, wears, boots, shoes, bikini}\}$
and set $V=W$. Suppose $\mathscr{B}=\{n_s,n_p,n,s\}$ and $\mathscr{B}'=\{n,s\}$ so that $J=\mathscr{C}(\mathscr{B})$ and $J'=\mathscr{C}(\mathscr{B}')$,
and let $T=(j,\alpha)$ be a translation from $F$ to $F'$,
where the language model $F$ has $F(n)=N$ and $F(n_s)=F(n_p)=N'$
where $N\cong \mathbb{R}^3$ is generated by $\{\vect{boots}, \vect{shoes}, \vect{bikini}\}$
and $N'=N\times \mathbb{R}$ where the extra dimension records the quantity conveyed by the noun. Let $F(s)=S$ be a one-dimensional space spanned by $\vect{1}$, which denotes \textit{surprise}.  For the purposes of this example we will normalize all non-zero values in $S$ to the vector $\vect{1}$. This gives only two attainable values in $S$; $\vect{s} = \vect{1}$ meaning that the sentence is surprising, and $\vect{s} = \vect{0}$ meaning that the sentence is not surprising. The language model $F'$ agrees with $F$ on both $n$ and $s$, that is $F'(n)=N$ and $F'(s)=S$. The functor $j$ is given by $j(n)=j(n_s)=j(n_p)=n$ and $j(s)=s$ and finally, the components of $\alpha$ are by the identity on every space except for $N'$ where it is defined as the canonical projection onto the first three coordinates.

\[ 
\begin{tikzcd}
W^* \arrow[d, "\ell"] & W^* \arrow[d, "\ell'"]\\
\PS(F) \arrow[r, "\PS(T)"'] &\PS(F')
\end{tikzcd}
	 \]
 defined as follows:
\[
\begin{tikzcd}[row sep=tiny]
W \arrow[r, "\ell"] & \PS(F) \\
\lang{Rosie} \arrow[r, maps to] & {(n_s,(2,5,3, 1)^T)} \\
\lang{boots} \arrow[r, maps to] & {(n_p,(1,0,0,2)^T)} \\
\lang{a~boot} \arrow[r, maps to] & {(n_p,(1,0,0,1)^T)} \\
\lang{wears} \arrow[r, maps to] & {\left(n_m^r s n_m^l, \begin{pmatrix} 1& 1& 1 & 0\\
																	-1 & -1 & -1 & 0\\
																	1 & 1& 1 & 0\\
																	-2& -2 &-1 & 1 \end{pmatrix} \right)}
\end{tikzcd}
\begin{tikzcd}[row sep=tiny]
W \arrow[r, "\ell'"] & \PS(F') \\
\lang{Rosie} \arrow[r, maps to] & {(n,(2,5,3)^T)} \\
\lang{boots} \arrow[r, maps to] & {(n,(1,0,0)^T)} \\
\lang{wears} \arrow[r, maps to] & {\left(n^r s n^l, \begin{pmatrix} 1& 1& 1\\
																	-1 & -1 & -1\\
																	1 & 1& 1\\ \end{pmatrix} \right)} \\
\end{tikzcd}
\]
where we use $n_m$ to denote either of $n_s$ or $n_p$. The type of $\lang{wears}$ is polymorphic - it can take singular, plural, or mass nouns as subject and object. Here, we consider types with singular or plural arguments.
	
The functor $\PS(T)$ uses $j$ and $\alpha$ to assign a tuple on the right to each to each tuple on the left, i.e. 
\[
\begin{tikzcd}[row sep=tiny]
\PS(F) \arrow[r, "\PS(T)"] & \PS(F') \\
{(n_s,(2,5,3, 1)^T)} \arrow[r, maps to] & {(n,(2,5,3)^T)} \\
{(n_p,(1,0,0,2)^T)} \arrow[r, maps to] & {(n,(1,0,0)^T)} \\
{\left(n_m^r s n_m^l, \begin{pmatrix} 1& 1& 1 & 0\\
																	-1 & -1 & -1 & 0\\
																	1 & 1& 1 & 0\\
																	-2& -2 &-1 & 1 \end{pmatrix} \right)}\arrow[r, maps to] & {\left(n^r s n^l, \begin{pmatrix} 1& 1& 1\\
																	-1 & -1 & -1\\
																	1 & 1& 1\\ \end{pmatrix} \right)} \\
\end{tikzcd}
\]
In particular, each item on the right hand side is of the form $(\PS(T)\circ \ell)(u)$ where $u$ is an element of $W$.
Now, in $F$, we have:
	\begin{align*}
	\vect{Rosie~wears~boots} &= F(\epsilon_n \cdot 1_s \cdot \epsilon_n)(\vect{Rosie} \otimes \vect{wears} \otimes \vect{boots})\\
	& = (\epsilon_n \otimes 1_s \otimes \epsilon_n)\left(\begin{pmatrix} 2\\5\\3\\1 \end{pmatrix} \otimes \begin{pmatrix} 1& 1& 1 & 0\\
									-1 & -1 & -1 & 0\\
									1 & 1& 1 & 0\\
									-2& -2 &-1 & 1 \end{pmatrix} \otimes \begin{pmatrix} 1\\0\\0\\2 \end{pmatrix}\right)= 0 \text{, i.e. unsurprising}
	\end{align*}
On the other hand,
		\begin{align*}
	&\vect{Rosie~wears~a~boot} = F(\epsilon_n \cdot 1_s \cdot \epsilon_n)(\vect{Rosie} \otimes \vect{wears} \otimes \vect{a~boot})\\
	& \qquad = (\epsilon_n \otimes 1_s \otimes \epsilon_n)\left(\begin{pmatrix} 2\\5\\3\\1 \end{pmatrix} \otimes \begin{pmatrix} 1& 1& 1 & 0\\
									-1 & -1 & -1 & 0\\
									1 & 1& 1 & 0\\
									-2& -2 &-1 & 1 \end{pmatrix} \otimes \begin{pmatrix} 1\\0\\0\\1 \end{pmatrix}\right)= -1 = 1 \text{ after normalization, i.e. surprising}
	\end{align*}	
We translate $\lang{Rosie wears boots}$ by computing 
\begin{align*}
&(\PS(T) \circ \ell)(\lang{Rosie wears boots}) = \PS(T)((n_s, \vect{Rosie}) \otimes (n_m^r s n_m^l, \vect{wears}) \otimes (n_p, \vect{boots})\\
&\qquad= ((j(n_s), \alpha_{n_s}(\vect{Rosie})) \otimes (j(n_m^r s n_m^l), \alpha_{n_m^rsn_m^l}(\vect{wears})) \otimes (j(n_p), \alpha_{n_p}(\vect{boots}))\\
&\qquad= \left(n, \begin{pmatrix}2\\ 5\\ 3 \end{pmatrix}\right) \otimes \left(n^r s n^l, \begin{pmatrix} 1& 1& 1\\
																	-1 & -1 & -1\\
																	1 & 1& 1\\ \end{pmatrix} \right) \otimes \left(n, \begin{pmatrix}1\\ 0\\ 0 \end{pmatrix}\right)= (\PS(T) \circ \ell)(\lang{Rosie wears a boot}) 																
\end{align*}
and applying the relevant reduction morphisms, we obtain:
	\begin{align*}
	\vect{Rosie~wears~boot}' &= F'(\epsilon_n \cdot 1_s \cdot \epsilon_n)(\vect{Rosie}' \otimes \vect{wears}' \otimes \vect{boot}')\\
	& = (\epsilon_n \otimes 1_s \otimes \epsilon_n)\left(\begin{pmatrix} 2\\5\\3 \end{pmatrix} \otimes \begin{pmatrix} 1& 1& 1\\
									-1 & -1 & -1\\
									1 & 1& 1 \end{pmatrix} \otimes \begin{pmatrix} 1\\0\\0 \end{pmatrix}\right)= 0 \text{, i.e. unsurprising}
	\end{align*}
	To translate the sentence \lang{Rosie wears a boot} we perform the same matrix multiplication to obtain a value of $0$ as well. In the language model $F$, these two sentences had distinct meanings. However, because $F'$ cannot detect the quantity of a noun, their translations are both unsurprising.

\end{ex}

This example shows how we can map a language with one grammatical structure onto another with a differing grammatical structure. In this case we have simplified the grammar, but we could also provide a translation $(j, \alpha)$ that maps the simpler grammar into the more complex grammar by identities and inclusion. The phenomenon of grammatical simplification is one that has been observed in various languages \cite{vantrijp2013, kroch1989}. This provides us with the beginnings of a way to describe these kinds of language evolution.
\subsection{Translating Between English and Spanish}\label{EngSpan}

In this section we construct a partial translation from English and Spanish. The relationship between English grammar and Spanish grammar is not functional; there are multiple types in Spanish that a single type in English should correspond to and vice versa.

Let $J_E = \mathscr{C} ( \{a_E, n_E\})$ be the grammar category for English and let $J_S = \mathscr{C}(\{a_S,n_S\})$ be the grammar category for Spanish. In English, adjectives multiply on the left to get a reduction
\[r: a_E n_E \to n_E \] and in Spanish adjectives multiply on the right to get a reduction 
\[q: n_S a_S \to n_S\]

Suppose there is a strict monoidal functor $j \maps J_E \to J_S$ which makes the assignment $j(a_E) = a_S$. We also wish to map the reduction $r$ to the reduction $q$. This requires that $j(a_E n_E) = n_S a_S$. By monoidality this means that $j(a_E) = n_S$ and $j(n_E) = a_S$. A monoidal functor cannot capture this relationship because it must be single valued. However, if we choose to only translate either adjectives or nouns we can construct a translation.

\begin{ex}[Translation at the phrase level]
	\label{ex:plt}
	In this example we choose to translate the fragment of English and Spanish grammar which includes nouns but not adjectives. We can also translate intransitive verbs from English to Spanish while keeping the functor between grammar categories single-valued.
	
Let $J_{En} = \C(\{n_E, s_E\})$ be the free compact closed category on the noun and sentence types in English and let $J_{S}$ be the isomorphic category $\C\{n_S, s_S\}$  generated by the corresponding types in Spanish.	Consider distributional categorical language models
\[F_{En} \maps J_{En} \to \Cat \text{ and } F_{Sp}\maps J_{Sp} \to \Cat\]
for English and Spanish respectively.
Consider a fragment of these languages consisting of only nouns and intransitive verbs. Let $F_{En}(n) = N_{En}$, $F_{En}(s) = S_{En}$, $F_{Sp}(n) = N_{Sp}$ and $F_{Sp}(s) = S_{Sp}$. Lexicons for the two languages can be populated by learning representations from text.
	
	To specify the translation $\PS(T)$ we set $j$ to be the evident functor which sends English types to their corresponding type in Spanish. To define a natural transformation $\alpha: F_{En} \rightarrow F_{Sp} \circ j$ it suffices to define $\alpha$ on the basic grammatical types which are not the nontrivial product of any two other types.
	Because $\alpha$ is a monoidal natural transformation, we have that $\alpha_{gh} = \alpha_g \otimes \alpha_h$ for every product type $gh$.
	
 If there were only one grammatical type $g$, then the language models would have no grammatical content and the translation would consist of a single linear transformation between words in English to words in Spanish. Learning this transformation is in fact a known method for implementing word-level machine translation, as outlined in \cite{mikolovtranslation, joulin2018}. 
 
 However, in general we need the natural transformation $\alpha$ to commute with the type reductions in $\C(\{n_E, s_E\})$.
	Indeed, consider $\vect{dogs} \in N_{En}$, $\vect{run} \in N_{En} \otimes S_{En}$, $\vect{perros} \in N_{Sp}$, $\vect{corren} \in N_{Sp} \otimes S_{Sp}$. We require that 
	
	\[ 
	\xymatrixcolsep{5pc}\xymatrix{ N_{En} \otimes N_{En} \otimes S_{En} \ar[d]^{\alpha_{n\cdot n^l \cdot s}} \ar[r]^-{\epsilon_{N_{En}} \otimes 1_{En}} & S_{En}\ar[d]^{\alpha_s} \\
		N_{Sp} \otimes N_{Sp} \otimes S_{Sp}  \ar[r]_-{\epsilon_{N_{Sp}} \otimes 1_{Sp}}& S_{Sp}	}
	\]
	commutes i.e. that if we first reduce $\vect{dogs} \otimes \vect{run}$ to obtain $\vect{dogs\;run}$ and then translate to $\vect{perros\;corren}$, we get the same as if we translate each word first, sending $\vect{dogs} \otimes \vect{run}$ to $\vect{perros} \otimes \vect{corren}$ and then reduce to $\vect{perros\;corren}$. Because these meaning reductions are built using dot products, this requirement is equivalent to the components of $\alpha$ being unitary linear transformations. In general, a linear transformation learned from a corpora will not be unitary. In this case we can replace $\alpha_g$ with the unitary matrix which is closest to it. This is a reasonable approximation because translations should preserve relative similarities between words in the same language.
\end{ex}

\section{Future Work}
We have defined a category $\DisCoCat$ which contains the categorical compositional distributional semantics of \cite{coecke2010} as objects and ways in which they can change as morphisms. We then outlined how this category can be used to translate, update or evolve different distributional compositional models of meaning.

There is a wide range of future work on this topic that we would like to explore. Some of the possible directions are the following:
\begin{itemize}

\item In this paper, we failed to construct a complete translation from English to Spanish using the definitions in this paper. The difficulty arose from the lack of a functional relationship between the two languages. To accommodate this, translations between language models can be upgraded by replacing functors with profunctors. This would include replacing the grammar transformation $j$ with a monoidal profunctor between the grammar categories. Because relationships between semantic meanings are also multi-valued we plan on replacing the components of the natural transformation with profunctors as well. 

\item This model can be improved to take advantage of the metric space structure of vector spaces to form dictionaries in a less trivial way. This would give a more intelligent way of forming translation dictionaries between two languages 

\item Whilst we have taken care to ensure that the category of language models we use is monoidal, we have not yet taken advantage of the diagrammatic calculus that is available to us. This is something we would like to do in future work.

\item We can better understand how language users negotiate a shared meaning space as in \cite{steels2003}, by modeling this as translations back and forth between agents. This will enhance the field of evolutionary linguistics by giving a model of language change that incorporates categorial grammar. 

\item We would like to use the methods here to implement computational experiments by creating compositional translation matrices for corpora in two different languages. These models of translation may also be used to make the previous computational experiments such as\cite{GrefenstetteSadrzadeh2011} \cite{KartsaklisSadrzadehPulman2012} more flexible.

\end{itemize}

\section{Acknowledgments}
Thanks go to Joey Hirsh for enlightening and fruitful discussions of the issues in this paper. We would also like to thank John Baez for catching a mistake in an early draft of this paper. We gratefully acknowledge support from the Applied Category Theory 2018 School at the Lorentz Center at which the research was conceived and developed, and funding from KNAW. Martha Lewis gratefully acknowledges support from NWO Veni grant `Metaphorical Meanings for Artificial Agents'.

\bibliographystyle{eptcs}
\bibliography{act}
\end{document}